\documentclass{article}

\PassOptionsToPackage{numbers,compress}{natbib}
 \usepackage[preprint]{neurips_2026}

\usepackage[utf8]{inputenc} 
\usepackage[T1]{fontenc}    
\usepackage{hyperref}       
\usepackage{url}            
\usepackage{booktabs}       
\usepackage{amsfonts}       
\usepackage{nicefrac}       
\usepackage{microtype}      
\usepackage{xcolor}         
\usepackage{amsmath,amssymb,amsthm}
\usepackage{multirow}
\usepackage{makecell}
\usepackage{algorithm}
\usepackage{algpseudocode}
\usepackage{graphicx}
\usepackage{soul}

\usepackage{xspace}
\newif\ifshowcomments
\showcommentstrue

\makeatletter
\let\todo\@undefined
\makeatother

\ifshowcomments
    \usepackage[textsize=scriptsize,textwidth=2.4cm]{todonotes}
\else
    \usepackage[disable,textsize=scriptsize,textwidth=2.4cm]{todonotes}
\fi

\title{ORCE: Order-Aware Alignment of Verbalized Confidence in Large Language Models}

\author{
Chen Li$^{1}$ \quad
Xiaoling Hu$^{2}$ \quad
Songzhu Zheng$^{3}$ \quad
Jiawei Zhou$^{1}$ \quad
Chao Chen$^{1}$ \\[3pt]
$^{1}$Stony Brook University, NY, USA \\
$^{2}$Massachusetts General Hospital and Harvard Medical School, MA, USA \\
$^{3}$Morgan Stanley, NY, USA \\
}

\begin{document}

\maketitle
\renewcommand{\thefootnote}{\fnsymbol{footnote}}
\footnotetext{Email: Chen Li (li.chen.8@stonybrook.edu).}
\renewcommand{\thefootnote}{\arabic{footnote}}
\begin{abstract}
Large language models (LLMs) often produce answers with high certainty even when they are incorrect, making reliable confidence estimation essential for deployment in real-world scenarios. Verbalized confidence, where models explicitly state their confidence in natural language, provides a flexible and user-facing uncertainty signal that can be applied even when token logits are unavailable. However, existing verbalized-confidence methods often optimize answer generation and confidence generation jointly, which can cause confidence-alignment objectives to interfere with answer accuracy. In this work, we propose a decoupled and order-aware framework for verbalized confidence calibration. Our method first generates an answer and then estimates confidence conditioned on the fixed question--answer pair, allowing confidence optimization without directly perturbing the answer-generation process. To align confidence with correctness likelihood, we construct a sampling-based surrogate from multiple model completions and optimize rank-based reinforcement learning objectives that encourage responses with higher estimated correctness likelihood to receive higher verbalized confidence. Experiments on reasoning and knowledge-intensive benchmarks show that our method improves calibration and failure prediction performance while largely preserving answer accuracy. These results demonstrate that verbalized confidence can be more reliably aligned by decoupling confidence estimation from answer generation and optimizing the relative ordering of confidence across responses.
\end{abstract}

\section{Introduction}
Large language models (LLMs) are increasingly deployed in high-stakes domains where the quality and reliability of generated answers are critical, such as law~\citep{li2025legalagentbench} and healthcare~\citep{li2024mediq}. In these settings, correctness alone is not enough: models must also know \textit{when and how to express uncertainty}. Yet current LLMs are not explicitly optimized to produce reliable confidence estimates, and their verbalized confidence often remains high even when their responses are incorrect~\citep{ji2023survey,sahoo2024comprehensive,huang2025survey,xiong2024can}. Such overconfidence undermines trust building and remains a key obstacle for the deployment of LLMs in safety-critical applications.
A natural approach to confidence estimation is to use the model's internal probabilities or logits assigned to generated answers~\citep{azaria2023internal,kadavath2022language,kapoor2024large}. Because LLMs generate text autoregressively, such probabilities are defined over tokens conditioned on previous tokens, making logit-based confidence a form of \emph{tokenized confidence}. However, tokenized confidence measures the likelihood of a specific token sequence rather than the semantic correctness of the response. This is poorly suited to free-form generation, where semantically equivalent answers may have different wordings and therefore different token likelihoods.
 They are also unavailable for many black-box LLMs. Verbalized confidence, in which the model explicitly states its confidence in natural language, provides a flexible and user-facing alternative that applies across model types and task formats. Yet existing studies show that verbalized confidence is frequently miscalibrated, overconfident, and highly sensitive to prompting. Prompt-based strategies~\citep{tian2023just,xiong2024can} can partially reduce overconfidence, but their improvements are limited, motivating explicit optimization of stated confidence. Separately, ensemble-style methods estimate uncertainty from group-level signals such as consistency across multiple sampled responses~\citep{xiong2024can}; however, they require repeated generation at inference time and can introduce additional variance from stochastic sampling. These limitations motivate a method that preserves the user-facing benefits of verbalized confidence while aligning it more directly with response correctness. 

Recent methods have attempted to improve verbalized confidence through supervised fine-tuning (SFT)~\cite{liconftuner}, or reinforcement learning (RL)~\cite{xu2024sayself, stengel2024lacie}. However, existing approaches jointly optimize answer generation and confidence generation. This coupling creates a crucial conflict:
objectives designed to improve uncertainty estimation may alter the model's answer distribution and degrade task accuracy (Appendix F.6 of~\cite{liconftuner}). In other words, improving confidence calibration should not come at the cost of generating incorrect answers. 
This issue is particularly problematic for user-facing LLM systems, where confidence estimates are only useful if they accompany strong answers.

In this work, we propose a simple but effective decoupled framework for verbalized confidence estimation 
Instead of training the model to generate answers and confidence jointly, we decouple the process into two stages: the model first produces an answer, and then generates a confidence estimate conditioned on the fixed question--answer pair. This design treats the generated answer as an input for confidence estimation, allowing us to optimize confidence behavior without directly perturbing the answer-generation process. As a result, our method improves the reliability of verbalized confidence while better preserving the model's original task accuracy.

The method is built on two key ideas. First, we estimate the uncertainty of the answer model through a \emph{sampling-based correctness likelihood}. For each prompt, we sample multiple completions from the answer model and use their empirical correctness frequency as a surrogate for how likely the model is to answer that prompt correctly. This provides a model-agnostic signal of answer reliability without requiring access to internal logits or token probabilities.
Second, we use this surrogate to train verbalized confidence with an \emph{order-aware} objective. Rather than only matching absolute confidence values, we optimize rank-based objectives that encourage the confidence model to preserve the relative ordering induced by the surrogate: responses associated with higher estimated correctness likelihood should receive higher verbalized confidence. This formulation is well suited to confidence estimation, where relative reliability is often more robust and practically more meaningful than exact numerical calibration.We evaluate our method on a diverse set of knowledge-intensive reasoning, discrete reasoning, and logical reasoning tasks. Across these settings, our method yields promising improvements on both calibration and failure-prediction metrics.


Our contributions are summarized as follows.
\begin{enumerate}
\item We identify the coupling between answer generation and confidence generation as an important limitation of existing verbalized-confidence alignment methods, and propose a two-stage framework that decouples the two. 
\item We introduce a sampling-based surrogate of correctness likelihood for supervising verbalized confidence without relying on internal model probabilities.
\item We propose an order-aware reinforcement learning objective that aligns verbalized confidence with this surrogate through rank-based optimization.
\item We show that our method improves confidence calibration and failure prediction performance across reasoning and knowledge-intensive benchmarks while largely preserving answer accuracy. 

\end{enumerate}
Together, these results suggest that verbalized confidence can be made more reliable by optimizing not only what confidence value the model states, but also how confidence is ordered across responses of different reliability.

\section{Related Work}
\paragraph{Confidence estimation in classical deep learning.}
Confidence estimation is widely studied~\cite{gal2016dropout, lakshminarayanan2017simple, guo2017calibration, nguyen2015deep, naeini2015obtaining, moon2020confidence, geifmanbias, mandelbaum2017distance, jungo2019assessing,zadrozny2002transforming} in classical deep learning, such as image classification and segmentation tasks. Surrogate based methods~\cite{moon2020confidence, liconfidence} use the correctness frequency or consistency as a surrogate for confidence estimation. Dropout~\cite{gal2016dropout} approximates predictive uncertainty by performing multiple stochastic forward passes with dropout enabled and measuring the mean prediction and variability across the resulting outputs. Deep ensembles~\cite{lakshminarayanan2017simple} estimate confidence by training multiple independently initialized neural networks and aggregating their predictions to obtain both the final prediction and its uncertainty. Because neural networks produce logit or probability outputs, classical confidence estimation methods can conveniently derive confidence scores directly from the model’s predictive distribution.

\paragraph{Confidence estimation in LLMs.} Prior work estimates LLM confidence from answer logits or probabilities~\citep{azaria2023internal,kadavath2022language,kapoor2024large}. While these internal scores are useful, they are tied to the likelihood of specific token sequences and are difficult to define for free-form responses with many semantically equivalent surface forms. Verbalized confidence offers a more flexible and user-facing alternative: it can be elicited without logit access, applies to black-box models, and more directly communicates uncertainty over the semantic correctness of generated answers.

Another line of work estimates confidence through self-consistency or ensembles by sampling multiple completions and aggregating agreement across responses~\citep{xiong2024can,wangself}. These methods are model-agnostic and capture uncertainty over the answer distribution, but require multiple generations at inference time and do not directly produce a single user-facing confidence statement.


Recent studies investigate \emph{verbalized confidence}, where LLMs explicitly state their confidence in natural language~\citep{linteaching,tian2023just,xiong2024can,stengel2024lacie}. Prompting-based studies, which ask language models to explicitly report how confident they are in their answers, show that verbalized confidence can be informative, but is often overconfident, prompt-sensitive, and imperfectly aligned with model uncertainty~\citep{xiong2024can,yona2024can}. Beyond prompting, LACIE~\citep{stengel2024lacie} calibrates confidence through listener-aware preference optimization, SaySelf~\citep{xu2024sayself} trains models to generate confidence with self-reflective rationales, ConfTuner~\citep{liconftuner} optimizes a tokenized Brier-score objective, and CONQORD~\citep{tao2024trust} aligns verbalized confidence with response quality using reinforcement learning. In contrast to these methods, which typically optimize answer and confidence generation jointly, our approach decouples the two stages: the model first generates an answer and then estimates confidence conditioned on the fixed response. This design enables order-aware alignment between verbalized confidence and estimated correctness likelihood while reducing interference with answer-generation accuracy.

\newtheorem{proposition}{Proposition}

\newtheorem{definition}{Definition}
\newtheorem{theorem}{Theorem}
\newtheorem{corollary}{Corollary}
\newtheorem{assumption}{Assumption}

\begin{figure}[t]
    \centering
    \includegraphics[width=1.0\linewidth]{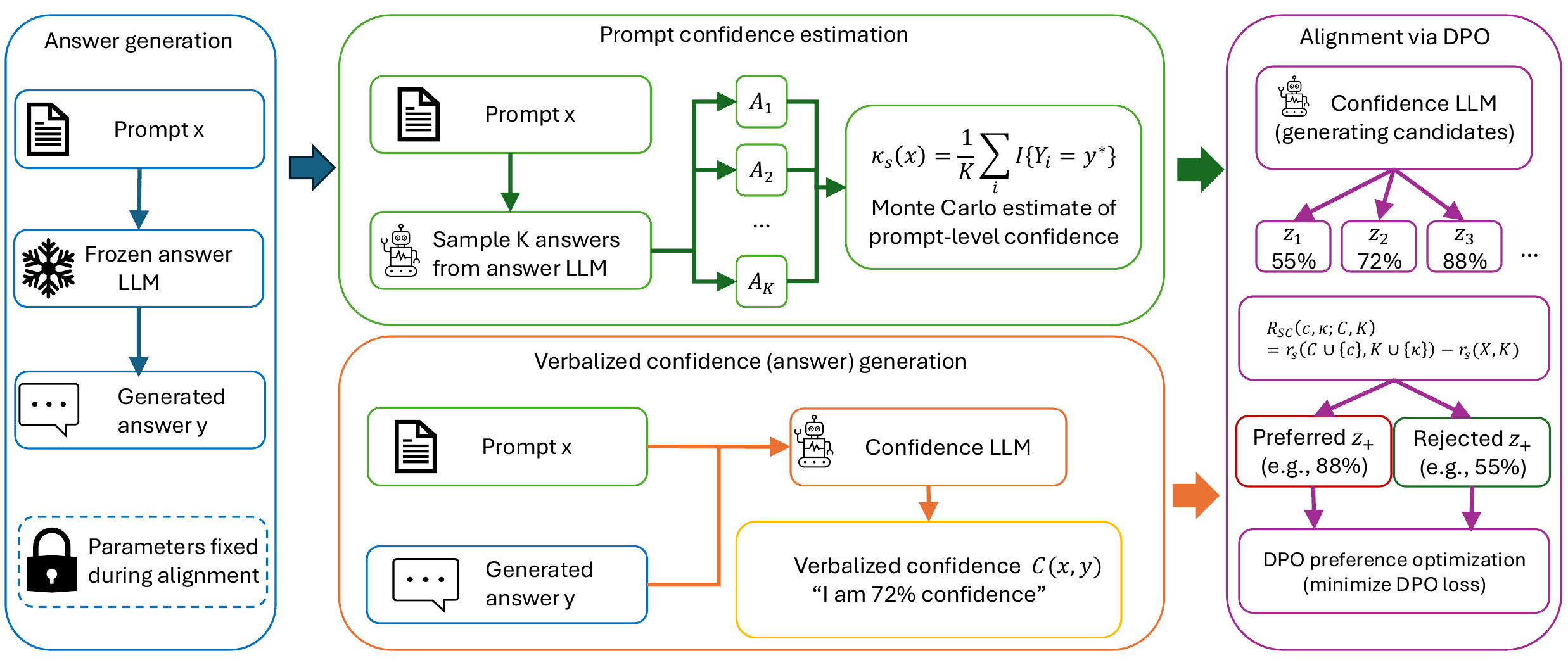}
    \caption{An overview of our method: ORCE separates answer generation from confidence alignment, uses an ensemble-based surrogate $\kappa_s(x)$, and optimizes the confidence LLM with a Spearman-correlation reward.}
    \label{fig:my_image}
\end{figure}

\section{Method}
We introduce ORCE, an order-aware reinforcement learning method for verbalized confidence estimation in language models. Our method decouples the answer and confidence generation into two stages. 
The overview of our method is shown in Figure~\ref{fig:my_image}.

\noindent\textbf{Decoupling answer and confidence generation.} 
In prior studies~\citep{xiong2024can, tao2024trust}, verbalized confidence is generated jointly with the answer using the same model: given a prompt, the LLM outputs both the answer and its associated confidence statement. In Miao et al.~\citep{xiong2024can}, this generation strategy is applied only at inference, making it simple and computationally efficient. However, when reinforcement learning (RL) fine-tuning is introduced~\citep{tao2024trust}, joint generation conflates two distinct objectives. When the policy seeks to improve outputs, the model cannot distinguish whether the improvement signal pertains to the answer or to the verbalized confidence, so RL updates intended to recalibrate confidence can inadvertently degrade answer accuracy. This issue is exacerbated by the limited size of confidence-alignment datasets relative to the pretraining corpus.

To address this, we decouple generation and confidence calibration into two separate stages. In the first stage, an answer LLM produces only the answer, so its parameters are never touched by confidence-alignment updates and answer accuracy is preserved by construction. 
In the second stage, a \textit{separate confidence LLM} generates a verbalized confidence conditioned on both the prompt and the previously generated answer. The confidence LLM is optimized via reinforcement learning. 

\noindent\textbf{From prompt-level confidence to answer-level confidence.} Once the Answer LLM is optimized, the Confidence LLM estimates the verbalized confidence $C(x,y)$ for each prompt-answer pair $(x,y)$. A primary challenge in training a reliable estimator is the absence of ground-truth labels; the model's internal certainty for a specific pair is not directly observable. However, we can approximate this at the prompt level by generating multiple answers for a single prompt and measuring their correctness against a ground-truth answer; such surrogate has been used in confidence estimation literature~\cite{moon2020confidence}.

We leverage these prompt-level confidences as coarse supervision for the Confidence LLM, which uses RL to align its verbalized confidence (at answer level). During training, the Confidence LLM generates confidence scores for a curated list of prompt-answer pairs. This list and the prompt-level ground truth confidences are both sorted and compared using the Spearman correlation (measuring their ranking similarity). To optimize this alignment, we employ DPO. Specifically, we iteratively generate new candidate answers and accept or reject the resulting pairs based on their ability to improve the ranking similarity, ultimately yielding a robust prompt-answer confidence estimator.

\subsection{Prompt Level Confidence Estimation}

For each prompt $x$, let $y^\star$ denote the reference answer, and let
$
g(y,y^\star)\in\{0,1\}
$
be a task-specific correctness function indicating whether an answer $y$ should be regarded as correct. For multiple-choice tasks, $g(y,y^\star)=\mathbb{I}\{y=y^\star\}$ reduces to exact match. For free-form tasks, $g$ is instantiated using the task's standard evaluation criterion, such as normalized exact match, answer extraction, or other benchmark-specific equivalence rules.

Using this correctness function, we can define the prompt-level reliability of the answer LLM as, 
$\eta(x):=\mathbb{P}\big(g(\hat Y,y^\star)=1 \mid x\big),$
where $\hat Y$ is the answer produced by the answer LLM. $\eta(x)$ measures the marginal probability that the answer model produces a correct answer on prompt $x$.

Since $\eta(x)$ is not directly accessible, we introduce an ensemble-based surrogate. For each prompt, we sample $K$ candidate answers $\{Y_i\}_{i=1}^{K}$ from the answer LLM at non-zero temperature and define, 
\begin{equation*}
    \kappa_s(x)=\frac{1}{K}\sum_{i=1}^{K} g(Y_i,y^\star).
\end{equation*}
By construction, $\kappa_s(x)$ is an unbiased Monte Carlo estimator of $\eta(x)$:
$
\mathbb{E}[\kappa_s(x)\mid x]=\eta(x),$
with variance $\eta(x)(1-\eta(x))/K$. Therefore, $\kappa_s(x)$ provides a consistent finite-sample estimate of the prompt-level reliability of the answer model.

\textbf{Selection of prompt-answer pairs.}
We generate prompt-answer pairs for the confidence LLM to estimate confidences, and to align them via DPO. The choice of these pairs, however, cannot be arbitrary.
After estimating $\kappa_s(x)$ from the $K$ sampled answers, we construct the realized answer according to the empirical majority-correctness regime. Concretely, if $\kappa_s(x)\ge 1/2$, we pair the prompt with a correct sampled answer; if $\kappa_s(x)<1/2$, we pair it with an incorrect sampled answer. In this way, the realized answer presented to the confidence model reflects the empirical correctness tendency of the answer model on that prompt, allowing the answer-conditioned score $C(x,y)$ to be trained against a prompt-level surrogate target.
\subsection{Order-aware reward via Spearman correlation}
\label{sec:sc_reward}

\textbf{Order-aware principle.}
We adopt an order-based criterion for verbalized confidence at the level of the constructed training distribution. For each prompt $x$, we first estimate its prompt-level reliability using the ensemble surrogate $\kappa_s(x)$, and then construct the realized answer $y$ so that its correctness state matches the majority-correctness regime implied by $\kappa_s(x)$. As a result, prompts with large surrogate reliability are paired with correct realized answers, whereas prompts with small surrogate reliability are paired with incorrect realized answers. Under this construction, we train the confidence model so that, across training instances,     $C(x_i,y_i) < C(x_j,y_j)
    \;\Longleftrightarrow\;
    \kappa_s(x_i) < \kappa_s(x_j)$,
up to finite-sample noise in the surrogate. Thus, although the confidence model is conditioned on realized answers, the ordering signal it is trained to preserve is induced by the prompt-level reliability of the answer model.


Given $n$ pairs of verbalized and surrogate confidences $\{(c_i,\kappa_i)\}_{i=1}^{n}$, let
$
    C=\{c_i\}_{i=1}^{n},  K=\{\kappa_i\}_{i=1}^{n},
$
and let $\mathrm{rank}(c_i)$ and $\mathrm{rank}(\kappa_i)$ denote the ranks of $c_i$ and $\kappa_i$ within $C$ and $K$, respectively.
A natural question is how to design a reward that encourages rank preservation. One simple local alternative is to compare the rank of each generated confidence value with the rank of its surrogate score, $R_{\mathrm{NRD}}(c_i,\kappa_i)
=
-\left|\mathrm{rank}(c_i)-\mathrm{rank}(\kappa_i)\right|.$
We refer to this baseline as Numeric Ranking Difference (NRD). Although intuitive, NRD is local: it scores each candidate only by its own rank discrepancy and does not directly evaluate how the candidate affects the overall monotone association between verbalized confidence and surrogate reliability. As a result, two candidates with similar individual rank discrepancies can have different effects on the global confidence--surrogate ordering.

To directly optimize the desired order-aware property, we instead use a global rank-correlation reward. Rather than matching each rank independently, a global reward scores a candidate by how much it improves or degrades the overall rank agreement between the verbalized confidence set and the surrogate set. Since our goal is to preserve the full reliability ordering induced by $\kappa_s$, we adopt a Spearman-correlation reward as our primary design.

\begin{figure}[t]
    \centering
    \includegraphics[width=1.0\linewidth]{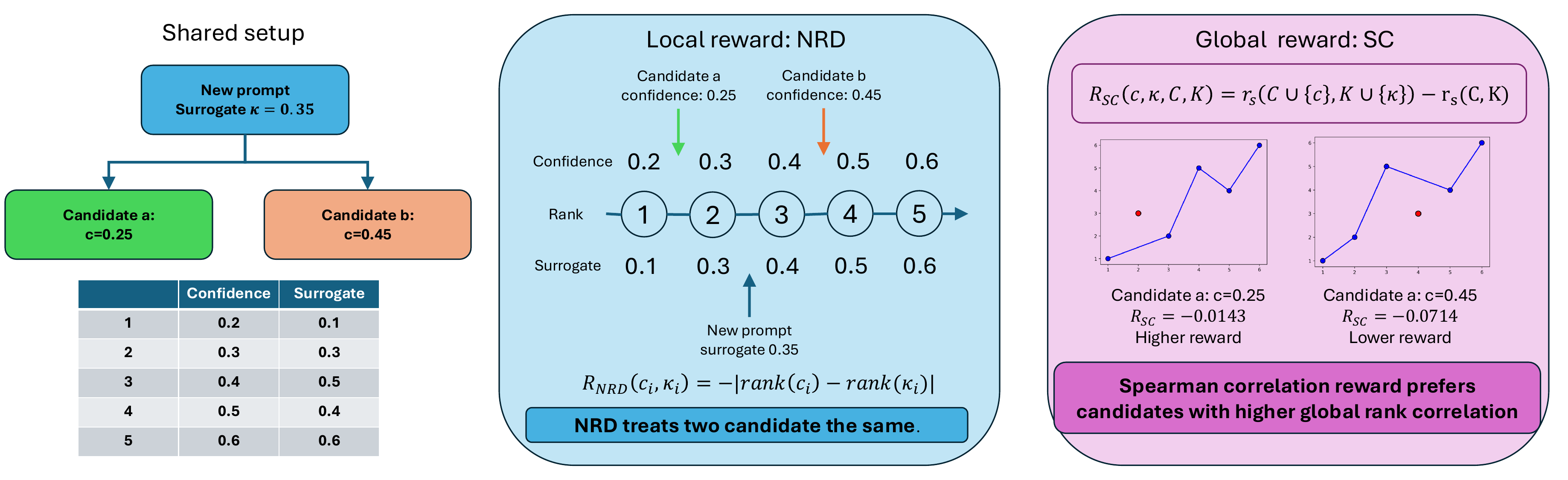}
    \caption{
Comparison between local rank matching (NRD) and global Spearman alignment. Although the two candidate confidences have similar local rank discrepancy under NRD, the SC reward prefers $c^{(a)}=0.25$ because it better preserves the global monotone relationship between verbalized confidence and surrogate reliability.
}
    \label{fig:cases}
\end{figure}

\textbf{Spearman correlation reward.}
Spearman's correlation measures the monotone association between two variables and therefore naturally matches our order-aware objective. The Spearman correlation between $C$ and $K$ is 
\begin{equation*}
    r_s(C,K)
    =
    \frac{
    \mathrm{cov}\!\left[\mathrm{rank}(C),\mathrm{rank}(K)\right]
    }{
    \sigma_{\mathrm{rank}(C)}\,\sigma_{\mathrm{rank}(K)}
    }.
\end{equation*}
For a new candidate pair $(c,\kappa)$, we define the Spearman-correlation (SC) reward as its marginal contribution to the global rank correlation, 
\begin{equation*}
    R_{\mathrm{SC}}(c,\kappa;C,K)
    =
    r_s\!\left(C\cup\{c\},K\cup\{\kappa\}\right)
    -
    r_s(C,K).
\end{equation*}
Thus, a candidate receives positive reward when adding it improves the monotone agreement between verbalized confidence and surrogate reliability, and negative reward when it disrupts this agreement. A comparison between SC and NRD reward is shown in Figure~\ref{fig:cases}.


\textbf{Warm start via SFT.}
A limitation of any correlation-based reward is that it can reinforce an incorrect monotone trend if the initial relationship between $C$ and $K$ is negatively correlated or degenerate. To prevent this, we perform a supervised fine-tuning (SFT) stage prior to reinforcement learning, which establishes a positive correlation between verbalized and surrogate confidences. This SFT pretrained model will produce $C$ for the reinforcement learning stage. As we formalize in Section~\ref{sec:theory}, this positivity is exactly the condition that ensures the SC reward points in the right direction during RL.

\subsection{Alignment strategy}
To align the confidence LLM with the surrogate using the SC reward, we adopt Direct Preference Optimization (DPO). DPO is well-matched to our setting: the SC reward naturally produces, for each prompt, a ranking over sampled verbalized confidences, from which preferred/rejected pairs can be extracted without on-policy rollouts or a separate value model.

The full alignment procedure is as follows. For each prompt in the training set, we (i) generate an answer using the answer LLM; (ii) sample multiple verbalized confidences from the confidence LLM for the resulting $(x, y)$ pair; (iii) score each sampled confidence $c$ using $R_{\mathrm{SC}}(c, \kappa_s(x); C, K)$, where $(C,K)$ is the reference set of confidence--surrogate pairs maintained during training; and (iv) designate the highest- and lowest-scoring samples as the preferred ($z^+$) and rejected ($z^-$) responses, respectively. Given the resulting preference data $(x, z^+, z^-)$, DPO encourages the model to assign higher likelihood to $z^+$ than to $z^-$ while remaining close to a reference model $\pi_{\mathrm{ref}}$ (the SFT confidence LLM prior to RL):
\begin{equation*}
    L_{\mathrm{DPO}} = -\mathbb{E}\!\left[\log \mathrm{sig}\!\left(\beta\!\left[\log \pi_\theta(z^+\!\mid x) - \log \pi_\theta(z^-\!\mid x) - \log \pi_{\mathrm{ref}}(z^+\!\mid x) + \log \pi_{\mathrm{ref}}(z^-\!\mid x)\right]\right)\right],
\end{equation*}
where $\pi_\theta$ is the confidence LLM being optimized, $\mathrm{sig}(\cdot)$ is the sigmoid function, and $\beta$ controls the strength of preference alignment.

\subsection{Theoretical analysis of the SC reward}
\label{sec:theory}

We give a brief interpretation of the SC reward from three perspectives: its ideal population optimum, its surrogate approximation, and the role of the SFT warm-start.

Let $X$ be a random prompt, let $Y^\star$ be its reference answer, and let $\hat Y \sim \pi_A(\cdot\mid X)$ be the answer produced by the answer LLM. Define the correctness indicator $Z:=g(\hat Y,Y^\star)\in\{0,1\}$ and the prompt-level reliability, $\eta(X):=\mathbb P(Z=1\mid X).$
The confidence model takes the realized prompt--answer pair $(X,\hat Y)$ as input and outputs a scalar score $S(X,\hat Y):=f_\theta(X,\hat Y)\in\mathbb R$.

\begin{theorem}[Spearman optimum]
\label{thm:spearman-characterization}
Assume $\eta(X)$ has a continuous distribution on $(0,1)$ and $S(X,\hat Y)$ has a continuous distribution on $\mathbb R$. Then
$\rho_S\!\big(S(X,\hat Y),\eta(X)\big)\le 1$,
with equality if and only if there exists a strictly increasing function $T:(0,1)\to\mathbb R$ such that
$S(X,\hat Y)=T\!\big(\eta(X)\big)$ almost surely.
\end{theorem}

Theorem~\ref{thm:spearman-characterization} shows that, at the population level, optimizing Spearman alignment recovers a score that is a monotone transform of the prompt-level reliability. Although the confidence model is conditioned on the realized answer, the population-optimal ordering is induced by $\eta(X)$.

In practice, $\eta(X)$ is unavailable, so we optimize against the ensemble surrogate
$
\kappa_s^{(K)}(X)=\frac{1}{K}\sum_{i=1}^K g(Y_i,Y^\star),
$
where $Y_1,\dots,Y_K$ are i.i.d.\ samples from $\pi_A(\cdot\mid X)$. By the strong law of large numbers, $\kappa_s^{(K)}(X)\to \eta(X)$ almost surely as $K\to\infty$. Under mild continuity assumptions, this implies
\[
\rho_S\!\big(S(X,\hat Y),\kappa_s^{(K)}(X)\big)\to
\rho_S\!\big(S(X,\hat Y),\eta(X)\big),
\]
so SC alignment against the surrogate approaches SC alignment against the ideal prompt-level target.

Finally, the SC reward is defined relative to a reference set $(C,K)$ through
\[
\Delta_{\mathrm{SC}}(c,\kappa;C,K)
=
r_s(C\cup\{c\},K\cup\{\kappa\})-r_s(C,K).
\]
Hence it acts as an update signal for improving a \emph{global} rank objective around the current anchor set. Because the SFT-initialized reference set is generally imperfect, warm-start is important: when the initial alignment is positive and the misordering is moderate, SC tends to favor candidates that better preserve the surrogate-induced ordering; if the anchor is severely misaligned, the same relative mechanism can reinforce the wrong trend. Thus, the role of SFT is not to produce a perfect ranking, but to initialize the model in a sufficiently aligned regime where global rank refinement is reliable. Proofs and additional discussion are deferred to Appendix~\ref{app:proofs}.

\section{Experiment}
In this section, we evaluate the confidence estimation performance of our method on three widely used open-weight foundation models: Llama-3-8B~\citep{grattafiori2024llama}, Qwen3-8B~\citep{yang2025qwen3}, and Mistral-7B-v0.3~\citep{jiang2023mistral}. These models span different model families and training recipes, allowing us to assess whether the proposed decoupled and order-aware confidence alignment framework generalizes across diverse LLM backbones.

\paragraph{Benchmarks.}
We evaluate our method on three representative reasoning settings: MMLU~\citep{hendrycksmeasuring} for knowledge-intensive reasoning, DROP~\citep{dua2019drop} for discrete reasoning, and LogiQA 2.0~\citep{liu2023logiqa} for logical reasoning. For logical reasoning, we train on LogiQA 2.0 and additionally report out-of-domain generalization results on ReClor~\citep{yureclor}. Results on the LogiQA 2.0 test set are provided in Appendix~\ref{app:extra_res}, and further experimental details are given in Appendix~\ref{app:exper_details}.



\paragraph{Baselines.}
We compare against representative prompting- and training-based baselines. \textbf{Vanilla} uses our decoupled answer-conditioned prompting format without any training. \textbf{Self-Consistency}~\citep{xiong2024can} estimates confidence from agreement across multiple sampled responses. \textbf{Top-$k$}~\citep{tian2023just} prompts the model to consider multiple likely answers before reporting confidence. \textbf{CoT}~\cite{wei2022chain,xiong2024can} elicits confidence after chain-of-thought reasoning. \textbf{ConfTuner}~\citep{liconftuner} is adapted to our decoupled setting so that only the confidence generator is trained while the answer remains fixed. \textbf{SFT} is a supervised fine-tuning baseline trained directly on the same sampling-based surrogate used in our method. Together, these baselines cover prompting-based elicitation, sampling-based estimation, and fine-tuning-based confidence calibration.

\paragraph{Evaluation metrics.}
We assess verbalized confidence from two complementary perspectives: calibration and failure prediction. For calibration, we use \textbf{Expected Calibration Error (ECE)}~\cite{naeini2015obtaining}. For ranking-based failure prediction, we use \textbf{Spearman correlation}~\cite{spearman1904proof}, \textbf{Area Under the Risk--Coverage Curve (AURC)}~\cite{geifman2017selective}, and \textbf{Excess AURC (EAURC)}~\cite{geifmanbias}. More details are in the Appendix~\ref{app:exper_details}.


\begin{table*}[t]
\scriptsize
\setlength{\tabcolsep}{7pt}
\centering
\begin{tabular}{llrrrrrr}
\toprule
\multirow{2.5}{*}{\makecell[l]{\textbf{Foundation}\\ \textbf{Models}}} 
& \multirow{2.5}{*}{\textbf{Methods}} 
& \multirow{2.5}{*}{\textbf{ECE} $\downarrow$}  
& \multicolumn{2}{c}{\textbf{Spearman Correlation}} 
& \multirow{2.5}{*}{\textbf{AURC} $\downarrow$} 
& \multirow{2.5}{*}{\textbf{E-AURC} $\downarrow$} 
& \multirow{2.5}{*}{\textbf{Accuracy} $\uparrow$}  \\
\cmidrule(lr){4-5} 
&  &  & correlation $\uparrow$ & $p$\_value $\downarrow$ \\
\midrule
\multirow{7}{*}{LLAMA-3 8B}  
& Vanilla      & 0.170 & 0.034 & $4.682\times10^{-5}$   & 0.332 & 0.265 & \textbf{0.657} \\
& Consistency  & 0.185 & \underline{0.464} & 0.000 & 0.196 & 0.127 & 0.653 \\
& Top-k        & 0.181 & 0.149 & $1.236\times10^{-70}$  & 0.299 & 0.224 & 0.639 \\
& CoT          & 0.232 & 0.204 & $3.058\times10^{-132}$ & 0.313 & 0.231 & 0.624 \\
& ConfTuner    & 0.167 & 0.409 & 0.000 & 0.172 & 0.105 & \textbf{0.657} \\
& SFT          & \underline{0.068} & 0.433 & 0.000 & \underline{0.165} & \underline{0.098} & \textbf{0.657} \\
& \bf Ours-ORCE  & \textbf{0.025} & \textbf{0.477} & 0.000 & \textbf{0.156} & \textbf{0.089} & \textbf{0.657} \\
\midrule
\multirow{7}{*}{Qwen3 8B}  
& Vanilla      & 0.212 & 0.198 & $1.143\times10^{-123}$ & 0.175 & 0.139 & \textbf{0.749} \\
& Consistency  & 0.204 & 0.368 & 0.000 & 0.163 & 0.128 & 0.745 \\
& Top-k        & \underline{0.043} & 0.247 & $8.568\times10^{-195}$ & 0.177 & 0.136 & 0.726 \\
& CoT          & 0.137 & 0.393 & 0.000 & 0.140 & 0.102 & 0.738 \\
& ConfTuner    & 0.088 & 0.400 & 0.000 & \underline{0.121} & \underline{0.085} & \textbf{0.749} \\
& SFT          & 0.130 & 0.390 & 0.000 & 0.149 & 0.113 & \textbf{0.749} \\
& \bf Ours-ORCE  & \textbf{0.034} & \textbf{0.414} & 0.000 & \textbf{0.116} & \textbf{0.080} & \textbf{0.749} \\
\midrule
\multirow{7}{*}{Mistral 7B}  
& Vanilla      & 0.248 & 0.010 & 0.240 & 0.374 & 0.292 & \textbf{0.625} \\
& Consistency  & 0.274 & 0.371 & 0.000 & 0.275 & 0.193 & 0.623 \\
& Top-k        & 0.247 & 0.171 & $8.971 \times 10^{-93}$ & 0.347 & 0.261 & 0.614 \\
& CoT          & 0.284 & 0.126 & $5.432 \times 10^{-51}$ & 0.385 & 0.283 & 0.584 \\
& ConfTuner    & 0.084 & 0.348 & 0.000 & 0.224 & \underline{0.143} & \textbf{0.625} \\
& SFT          & \textbf{0.066} & 0.316 & 0.000 & \underline{0.228} & 0.147 & \textbf{0.625} \\
& \bf Ours-ORCE  & \underline{0.073} & \textbf{0.408} & 0.000 & \textbf{0.206} & \textbf{0.124} & \textbf{0.625} \\
\bottomrule
\end{tabular}
\caption{Alignment performance of methods across the foundation models on the MMLU dataset.}
\label{tab:mmlu}
\end{table*}

\begin{table*}[t]
\scriptsize
\setlength{\tabcolsep}{8pt}
\centering
\begin{tabular}{llrrrrrr}
\toprule
\multirow{2.5}{*}{\makecell[l]{\textbf{Foundation}\\ \textbf{Models}}} 
& \multirow{2.5}{*}{\textbf{Methods}} 
& \multirow{2.5}{*}{\textbf{ECE} $\downarrow$}  
& \multicolumn{2}{c}{\textbf{Spearman Correlation}} 
& \multirow{2.5}{*}{\textbf{AURC} $\downarrow$} 
& \multirow{2.5}{*}{\textbf{E-AURC} $\downarrow$} 
& \multirow{2.5}{*}{\textbf{F1} $\uparrow$} \\
\cmidrule(lr){4-5} 
&  &  & correlation $\uparrow$ & $p$\_value $\downarrow$ & & & \\
\midrule

\multirow{7}{*}{LLAMA-3 8B}  
& Vanilla & 0.327 & -0.024 & 0.021 & 0.459 & 0.319 & \underline{0.568}  \\
& Consistency & \textbf{0.038} & 0.514 & 0.000 & 0.272 & 0.123 & 0.551 \\
& Top-k & 0.373 & -0.120 & $3.848 \times 10^{-32}$ & 0.550 & 0.383 & 0.524 \\
& CoT & 0.221 & 0.065 & $1.588 \times 10^{-10}$ & 0.373 & 0.278 & \textbf{0.647} \\
& ConfTuner & 0.119 & \underline{0.571} & 0.000 & 0.243 & 0.102 & \underline{0.568} \\
& SFT & 0.093 & 0.568 & 0.000 & \underline{0.237} & \underline{0.097} & \underline{0.568} \\
& \bf Ours-ORCE & \underline{0.074} & \textbf{0.628} & 0.000 & \textbf{0.218} & \textbf{0.078} & \underline{0.568} \\

\midrule

\multirow{7}{*}{Qwen3 8B}  
& Vanilla & 0.314 & 0.012 & 0.248 & 0.472 & 0.328 & 0.568 \\
& Consistency & 0.266 & 0.413 & 0.000 & 0.312 & 0.180 & 0.579 \\
& Top-k & 0.310 & 0.176 & $4.371 \times 10^{-67}$ & 0.349 & 0.245 & \underline{0.626} \\
& CoT & 0.142 & 0.278 & $4.132 \times 10^{-169}$ & \textbf{0.199} & 0.161 & \textbf{0.777} \\
& ConfTuner & 0.114 & 0.531 & 0.000 & 0.275 & 0.131 & 0.568 \\
& SFT & \underline{0.049} & \underline{0.625} & 0.000 & 0.226 & \underline{0.082} & 0.568 \\
& \bf Ours-ORCE & \textbf{0.048} & \textbf{0.672} & 0.000 & \underline{0.214} & \textbf{0.070} & 0.568 \\

\midrule

\multirow{7}{*}{Mistral 7B}  
& Vanilla & 0.512 & 0.064 & $2.931 \times 10^{-10}$ & 0.469 & 0.300 & \underline{0.526} \\
& Consistency & 0.132 & \underline{0.523} & 0.000 & 0.312 & \underline{0.129} & 0.506 \\
& Top-k & 0.433 & 0.136 & $7.944 \times 10^{-41}$ & 0.467 & 0.283 & 0.505 \\
& CoT & 0.249 & 0.185 & $3.743\times 10^{-74}$ & 0.343 & 0.241 & \textbf{0.628} \\
& ConfTuner & \textbf{0.021} & 0.081 & $2.143\times 10^{-15}$ & 0.506 & 0.336 & \underline{0.526} \\
& SFT & 0.058 & 0.481 & 0.000 & \underline{0.305} & 0.135 & \underline{0.526} \\
& \bf Ours-ORCE & \underline{0.054} & \textbf{0.578} & 0.000 & \textbf{0.274} & \textbf{0.104} & \underline{0.526} \\

\bottomrule
\end{tabular}
\caption{Alignment performance of methods across the foundation models on the DROP dataset.}
\label{tab:drop}
\end{table*}
\paragraph{Results on MMLU.}
Table~\ref{tab:mmlu} reports the calibration and failure prediction performance on MMLU across three foundation models. Overall, Ours-ORCE consistently achieves the best or near-best calibration performance while preserving the original answer accuracy. In particular, it obtains the lowest ECE on LLAMA-3 8B and Qwen3 8B, reducing ECE from 0.170 to 0.025 and from 0.212 to 0.034 compared with Vanilla, and substantially outperforming ConfTuner on both models. On Mistral 7B, Ours-ORCE achieves the strongest failure prediction performance, with the highest Spearman correlation, lowest AURC/E-AURC, and best AUPR among all methods, while maintaining the same accuracy as Vanilla. Compared with prompting-based baselines such as Top-k and CoT, Ours-ORCE provides much stronger calibration and ranking performance, suggesting that prompting alone is insufficient to reliably align verbalized confidence. 

\paragraph{Results on DROP.}
Table~\ref{tab:drop} reports the results on DROP. Overall, \textbf{Ours-ORCE} achieves the strongest confidence-alignment performance across the three foundation models while preserving the original answer-generation performance. Compared with vanilla models, Ours-ORCE substantially reduces ECE and improves Spearman correlation, AURC, E-AURC, and AUPR, indicating that the learned verbalized confidence is better aligned with correctness likelihood and more useful for failure prediction. For LLAMA-3 8B, Ours-ORCE obtains the best Spearman correlation, AURC, and E-AURC, and also improves ECE over ConfTuner and SFT. For Qwen 7B, Ours-ORCE achieves the best results on nearly all confidence metrics, including the lowest ECE, highest Spearman correlation, lowest AURC/E-AURC, and highest AUPR. For Mistral 7B, Ours-ORCE achieves the best Spearman correlation and failure prediction metrics, although ConfTuner obtains the lowest ECE; this suggests that pointwise calibration alone does not necessarily imply better ranking quality or failure prediction performance. 

\paragraph{Results on ReClor.}
Table~\ref{tab:reclor} reports results on ReClor, where Ours-ORCE consistently improves verbalized confidence alignment across all three foundation models while preserving answer accuracy. Compared with Vanilla, Ours-ORCE substantially reduces ECE from 0.178 to 0.065 on LLAMA-3 8B, from 0.164 to 0.044 on Qwen 7B, and from 0.468 to 0.061 on Mistral 7B. It also achieves the strongest Spearman correlation and the best AURC/E-AURC and AUPR across all models, indicating better ranking of correct versus incorrect responses and stronger failure prediction performance. Compared with ConfTuner and SFT, Ours-ORCE further improves calibration and ranking metrics, showing the benefit of order-aware alignment beyond standard fine-tuning. 

\begin{table*}[t]
\scriptsize
\setlength{\tabcolsep}{7pt}
\centering
\begin{tabular}{llrrrrrr}
\toprule
\multirow{2.5}{*}{\makecell[l]{\textbf{Foundation}\\ \textbf{Models}}} 
& \multirow{2.5}{*}{\textbf{Methods}} 
& \multirow{2.5}{*}{\textbf{ECE} $\downarrow$}  
& \multicolumn{2}{c}{\textbf{Spearman Correlation}} 
& \multirow{2.5}{*}{\textbf{AURC} $\downarrow$} 
& \multirow{2.5}{*}{\textbf{E-AURC} $\downarrow$} 
& \multirow{2.5}{*}{\textbf{Accuracy} $\uparrow$}  \\
\cmidrule(lr){4-5} 
&  &  & correlation $\uparrow$ & $p$\_value $\downarrow$ & & & \\
\midrule

\multirow{7}{*}{LLAMA-3 8B}  
& Vanilla & 0.178 & 0.053 & 0.239 & 0.373 & 0.283 & \underline{0.608} \\
& Consistency & 0.376 & 0.072 & 0.110 & 0.355 & 0.269 & \textbf{0.616} \\
& Top-k & 0.270 & 0.061 & 0.177 & 0.441 & 0.298 & 0.516 \\
& CoT & 0.338 & 0.067 & 0.134 & 0.441 & 0.304 & 0.524 \\
& ConfTuner & 0.140 & 0.343 & $3.196\times 10^{-15}$ & 0.250 & 0.161 & \underline{0.608} \\
& SFT & \underline{0.084} & \underline{0.446} & $8.378 \times 10^{-26}$ & \underline{0.201} & \underline{0.112} & \underline{0.608} \\
& \bf Ours-ORCE & \textbf{0.065} & \textbf{0.544} & $8.061 \times 10^{-40}$ & \textbf{0.169} & \textbf{0.080} & \underline{0.608} \\

\midrule

\multirow{7}{*}{Qwen3 8B}  
& Vanilla & 0.164 & 0.059 & 0.187 & 0.190 & 0.167 & 0.794 \\
& Consistency & 0.167 & 0.311 & $1.206 \times 10^{-12}$ & 0.147 & 0.126 & \textbf{0.802} \\
& Top-k & \underline{0.046} & 0.133 & $2.797 \times 10^{-3}$ & 0.161 & 0.139 & \underline{0.800} \\
& CoT & 0.124 & 0.084 & 0.062 & 0.183 & 0.161 & 0.796 \\
& ConfTuner & 0.047 & 0.351 & $6.047 \times 10^{-16}$ & 0.091 & 0.069 & 0.794 \\
& SFT & 0.062 & \underline{0.537} & $9.431\times 10^{-39}$ & \underline{0.083} & \underline{0.060} & 0.794 \\
& \bf Ours-ORCE & \textbf{0.044} & \textbf{0.626} & $8.051 \times 10^{-56}$ & \textbf{0.064} & \textbf{0.041} & 0.794 \\

\midrule

\multirow{7}{*}{Mistral 7B}  
& Vanilla & 0.468 & 0.049 & 0.275 & 0.477 & 0.332 & 0.512 \\
& Consistency & 0.483 & 0.029 & 0.523 & 0.489 & 0.346 & \underline{0.514} \\
& Top-k & 0.357 & 0.153 & $6.216 \times 10^{-4}$ & 0.438 & 0.292 & 0.510 \\
& CoT & 0.300 & 0.124 & 0.006 & 0.389 & 0.271 & \textbf{0.556} \\
& ConfTuner & 0.141 & 0.191 & $1.712 \times 10^{-5}$ & 0.408 & 0.263 & 0.512 \\
& SFT & \underline{0.073} & \underline{0.344} & $2.368\times 10^{-15}$ & \underline{0.329} & \underline{0.184} & 0.512 \\
& \bf Ours-ORCE & \textbf{0.061} & \textbf{0.392} & $8.905 \times 10^{-20}$ & \textbf{0.306} & \textbf{0.161} & 0.512 \\

\bottomrule
\end{tabular}
\caption{Alignment performance of methods across the foundation models on the ReClor dataset.}
\label{tab:reclor}
\end{table*}

\paragraph{Ablation study.}
We conduct ablation studies on MMLU with LLAMA-3 8B to evaluate the contribution of each component in our method. First, we compare against a SaySelf-inspired reward, which selects high verbalized confidence for correct responses and low verbalized confidence for incorrect responses. As shown in Table~\ref{tab:ablation}, this objective achieves relatively low ECE, but performs worse on ranking-based metrics such as Spearman correlation, AURC, and E-AURC, suggesting that correctness-dependent confidence preference alone is insufficient to learn fine-grained confidence ordering. Second, we replace our rank-based reward with a pointwise numeric-distance objective between verbalized confidence and the surrogate confidence score. This variant also underperforms Ours-ORCE, indicating that directly matching absolute confidence values is less effective than preserving the relative reliability ordering among responses. Third, the W/o gen (SC) variant replaces the generated verbalized confidence list used in reward construction with the surrogate confidence list. This variant performs worse on all calibration and failure-prediction metrics, indicating that effective alignment requires explicitly optimizing the relative ordering of the model’s own generated confidence outputs, rather than relying only on the surrogate ranking. Finally, we compare our Spearman correlation reward with the Numeric Ranking Difference (NRD) reward. Although NRD improves over several ablations, Ours-ORCE achieves the best ECE, Spearman correlation, AURC, and E-AURC, demonstrating that global order-aware alignment is more effective than local rank-difference matching. Notably, all variants preserve the same answer accuracy, further supporting the benefit of decoupling confidence optimization from answer generation.

\begin{table*}[t]
\scriptsize
\setlength{\tabcolsep}{4pt}
\centering
\begin{tabular}{llrrrrrr}
\toprule
\multirow{2.5}{*}{\makecell[l]{\textbf{Foundation}\\ \textbf{Models}}} 
& \multirow{2.5}{*}{\textbf{Methods}} 
& \multirow{2.5}{*}{\textbf{ECE} $\downarrow$}  
& \multicolumn{2}{c}{\textbf{Spearman Correlation}} 
& \multirow{2.5}{*}{\textbf{AURC} $\downarrow$} 
& \multirow{2.5}{*}{\textbf{E-AURC} $\downarrow$} 
& \multirow{2.5}{*}{\textbf{Accuracy} $\uparrow$}  \\
\cmidrule(lr){4-5} 
&  &  & correlation $\uparrow$ & $p$\_value $\downarrow$ & & & \\
\midrule
\multirow{6}{*}{LLAMA-3 8B}  
& Vanilla  & 0.170 & 0.034 & $4.682\times10^{-5}$ & 0.332 & 0.265 & \textbf{0.657} \\
& SaySelf-inspired RL objective & \underline{0.038} & 0.449 & 0.000 & 0.178 & 0.111 & \textbf{0.657} \\
& W/o gen (SC) & 0.143 & 0.468 & 0.000 & 0.160 & 0.093 & \textbf{0.657} \\
& Numeric dist & 0.137 & 0.469 & 0.000 & 0.160 & 0.093 & \textbf{0.657} \\
& NRD & 0.110 & \underline{0.473} & 0.000 & \underline{0.159} & \underline{0.092} & \textbf{0.657} \\
& \bf Ours-ORCE & \textbf{0.025} & \textbf{0.477} & 0.000 & \textbf{0.156} & \textbf{0.089} & \textbf{0.657} \\
\bottomrule
\end{tabular}
\caption{Ablation study results of our method on MMLU dataset with LLAMA-3 8B.}
\label{tab:ablation}
\end{table*}

\begin{table*}[t]
\scriptsize
\setlength{\tabcolsep}{8pt}
\centering
\begin{tabular}{llrrrrrr}
\toprule
\multirow{2.5}{*}{\makecell[l]{\textbf{Foundation}\\ \textbf{Models}}} 
& \multirow{2.5}{*}{\textbf{Methods}} 
& \multirow{2.5}{*}{\textbf{ECE} $\downarrow$}  
& \multicolumn{2}{c}{\textbf{Spearman Correlation}} 
& \multirow{2.5}{*}{\textbf{AURC} $\downarrow$} 
& \multirow{2.5}{*}{\textbf{E-AURC} $\downarrow$} 
& \multirow{2.5}{*}{\textbf{Accuracy} $\uparrow$} \\
\cmidrule(lr){4-5} 
&  &  & correlation $\uparrow$ & $p$\_value $\downarrow$ & & & \\
\midrule

\multirow{4}{*}{Qwen3 8B}  
& Vanilla & 0.212 & 0.198 & $1.143\times10^{-123}$ & 0.175 & 0.139 & \textbf{0.749} \\
& ConfTuner & 0.290 & \underline{0.356} & 0.000 & 0.122 & 0.086 & \textbf{0.749} \\
& SFT & \underline{0.112} & 0.351 & 0.000 & \underline{0.118} & \underline{0.084} & \textbf{0.749} \\
& \bf Ours-ORCE & \textbf{0.049} & \textbf{0.397} & 0.000 & \textbf{0.107} & \textbf{0.073} & \textbf{0.749} \\

\midrule

\multirow{4}{*}{Mistral 7B}  
& Vanilla & 0.248 & 0.010 & 0.240 & 0.374 & 0.292 & \textbf{0.625} \\
& ConfTuner & 0.164 & \underline{0.458} & 0.000 & \underline{0.182} & \underline{0.100} & \textbf{0.625} \\
& SFT & \underline{0.054} & \underline{0.458} & 0.000 & \underline{0.182} & \underline{0.100} & \textbf{0.625} \\
& \bf Ours-ORCE & \textbf{0.025} & \textbf{0.517} & 0.000 & \textbf{0.169} & \textbf{0.087} & \textbf{0.625} \\

\bottomrule
\end{tabular}
\caption{Alignment performance of methods trained on LLAMA-3 8B and transferred to Qwen 7B and Mistral 7B.}
\label{tab:transfer}
\end{table*}

Table~\ref{tab:transfer} evaluates the transferability of confidence-alignment methods trained on LLAMA-3 8B to other foundation models. Ours-ORCE consistently achieves the best performance on both Qwen 7B and Mistral 7B, obtaining the lowest ECE, highest Spearman correlation, and lowest AURC/E-AURC across all transferred settings. Compared with ConfTuner and SFT, Ours-ORCE shows stronger calibration and failure prediction performance, suggesting that the proposed order-aware alignment objective transfers more effectively across model families. Importantly, all methods preserve the same answer accuracy within each target model, indicating that the transferred confidence estimator improves uncertainty alignment without affecting the underlying answer-generation behavior.

\section{Conclusion}
In this work, we propose a decoupled and order-aware framework for verbalized confidence estimation in large language models. By separating answer generation from confidence estimation, our method improves uncertainty estimation while preserving the model's answer-generation behavior. We further introduce a Spearman-based reinforcement learning objective that aligns verbalized confidence with a sampling-based surrogate of correctness likelihood, encouraging more reliable responses to receive higher stated confidence. Experiments across multiple reasoning and knowledge-intensive benchmarks show that our approach consistently improves calibration and failure prediction performance without degrading answer accuracy. These results suggest that reliable verbalized confidence requires not only calibrated confidence values, but also the correct ordering of confidence across responses with different levels of reliability.

\newpage

\bibliography{iclr2026_conference}
\bibliographystyle{iclr2026_conference}
\appendix

\newpage
\section{Additional theoretical details}
\label{app:proofs}

This appendix provides additional justification for the theoretical discussion in Section~\ref{sec:theory}. Our goal is to clarify three points: (i) why the ensemble surrogate $\kappa_s$ is a consistent approximation to the ideal prompt-level target $\eta(X)$; (ii) why the Spearman-correlation (SC) reward should be interpreted as a marginal improvement in a \emph{global} rank objective; and (iii) why a positive SFT warm-start is important when the reward is defined relative to an existing reference ranking. Because the practical training procedure uses a finite-sample surrogate, a drifting reference set, and DPO rather than exact ascent on a population objective, the results below characterize the \emph{idealized objective being approximated} rather than the full dynamics of the practical algorithm.

Throughout, $X \sim \mathcal{D}_X$ denotes a random prompt, $Y^\star$ its ground-truth answer, and $\hat Y \sim \pi_A(\cdot\mid X)$ the answer produced by the answer LLM. The prompt-level target confidence is
\[
  \eta(X) := \mathbb{P}(\hat Y = Y^\star \mid X),
\]
and the confidence model outputs a scalar score
\[
  S(X,\hat Y) := f_\theta(X,\hat Y) \in \mathbb{R}.
\]
We assume that $\eta(X)$ and $S(X,\hat Y)$ have continuous distributions, so that Spearman correlation is well-defined without ties.

\subsection{Consistency of the surrogate objective}
\label{app:proof-consistency}

We begin by justifying the use of the ensemble surrogate
\[
  \kappa_s^{(K)}(X) := \frac{1}{K}\sum_{i=1}^{K}\mathbb{I}\{Y_i = Y^\star\},
\]
where $Y_1,\dots,Y_K$ are i.i.d.\ samples from $\pi_A(\cdot\mid X)$.

\begin{proposition}[Pointwise consistency of the surrogate]
\label{prop:surrogate-pointwise}
For $\mathcal{D}_X$-almost every $x$,
\[
  \kappa_s^{(K)}(x) \xrightarrow{\mathrm{a.s.}} \eta(x)
  \qquad\text{as } K\to\infty.
\]
Consequently,
\[
  \kappa_s^{(K)}(X) \xrightarrow{\mathrm{a.s.}} \eta(X)
\]
as random variables on the joint probability space.
\end{proposition}

\begin{proof}
For each fixed $x$, the indicators $\mathbb{I}\{Y_i = y^\star\}$ are i.i.d.\ Bernoulli random variables with mean
\[
  \mathbb{P}(Y_i = y^\star \mid x)=\eta(x).
\]
The strong law of large numbers therefore gives
\[
  \kappa_s^{(K)}(x)\xrightarrow{\mathrm{a.s.}}\eta(x).
\]
Since this holds for $\mathcal{D}_X$-almost every $x$, the same conclusion holds for the random variable $X$.
\end{proof}

\begin{proposition}[Population-level consistency of Spearman alignment]
\label{prop:surrogate-spearman}
Assume the marginals of $\eta(X)$ and $S(X,\hat Y)$ are continuous. If the population Spearman functional is continuous at the joint law of $(S(X,\hat Y),\eta(X))$, then
\[
  \rho_S\!\big(S(X,\hat Y), \kappa_s^{(K)}(X)\big)
  \xrightarrow{}
  \rho_S\!\big(S(X,\hat Y), \eta(X)\big)
\]
as $K\to\infty$.
\end{proposition}

\begin{proof}[Proof sketch]
By Proposition~\ref{prop:surrogate-pointwise},
\[
  \kappa_s^{(K)}(X)\xrightarrow{\mathrm{a.s.}}\eta(X).
\]
Since $S(X,\hat Y)$ does not depend on $K$, this yields joint almost-sure convergence
\[
  \big(S(X,\hat Y),\kappa_s^{(K)}(X)\big)
  \xrightarrow{\mathrm{a.s.}}
  \big(S(X,\hat Y),\eta(X)\big).
\]
Hence the corresponding joint laws converge weakly. Under continuity of the population, Spearman's functional law at the limit, the claim follows by the continuous mapping principle.
\end{proof}

Proposition~\ref{prop:surrogate-spearman} formalizes the sense in which optimizing SC against the surrogate becomes equivalent, as $K$ grows, to optimizing SC against the ideal target $\eta(X)$.

\subsection{Why the SC reward is global}
\label{app:sc-global}

The SC reward used in training is
\[
  \Delta_{\mathrm{SC}}(c,\kappa;C,K)
  :=
  r_s(C\cup\{c\}, K\cup\{\kappa\}) - r_s(C,K),
\]
where $(C,K)$ is a current reference set. Unlike local rank-matching rewards, this quantity depends on how inserting $(c,\kappa)$ changes the rank correlation of the \emph{entire enlarged set}. This global dependence is the key reason SC can distinguish candidates with the same local rank error.

\begin{proposition}[SC as a marginal global objective]
\label{prop:sc-global-objective}
Let $(C,K)$ be a finite reference set with no ties, and let $(c,\kappa)$ be a candidate pair. Then $\Delta_{\mathrm{SC}}(c,\kappa;C,K)$ is exactly the marginal change in the global rank-correlation statistic $r_s$ induced by inserting the candidate. Consequently, two candidates can receive different SC rewards even when they have the same local rank error, provided they interact differently with the surrounding rank structure.
\end{proposition}

\begin{proof}
The first statement follows directly from the definition of $\Delta_{\mathrm{SC}}$. The second follows because the rank ordering of the enlarged set depends not only on the inserted point's own rank, but also on how the insertion changes the relative ordering of the existing points in the augmented ranking. Therefore, SC depends on the candidate's compatibility with the full reference structure, not only on its pointwise rank difference.
\end{proof}

Proposition~\ref{prop:sc-global-objective} captures the conceptual distinction between SC and local rewards such as NRD: NRD is a pointwise rank-matching objective, whereas SC is a global rank-consistency objective. 

\subsection{Why positive warm-start matters}
\label{app:warmstart}

The SC reward is defined relative to the current monotone trend encoded by $(C,K)$. This makes initialization important.

\begin{proposition}[Sign symmetry of rank correlation]
\label{prop:signsymmetry}
If a score function $S$ achieves Spearman correlation $\rho_S(S,\eta)=1$, then $-S$ achieves Spearman correlation $-1$. More generally, rank-correlation objectives alone do not distinguish the desired monotone solution from an anti-monotone one without an orientation condition.
\end{proposition}

\begin{proof}
If $S=T\circ \eta$ for a strictly increasing $T$, then $-S=(-T)\circ\eta$ with $-T$ strictly decreasing. Therefore, the rank ordering induced by $-S$ is exactly the reverse of that induced by $\eta$, implying Spearman correlation $-1$.
\end{proof}

Proposition~\ref{prop:signsymmetry} shows that an order-only objective must be oriented correctly: without some positive initial alignment, a rank-based update rule can, in principle, refine the wrong monotone trend just as consistently as the correct one.

\begin{proposition}[Interpretation of the warm-start condition]
\label{prop:warmstart-interpretation}
Suppose the SFT-initialized score $S_0$ satisfies
\[
  \rho_S\!\big(S_0(X,\hat Y),\eta(X)\big) > 0.
\]
Then the induced confidence ranking is already positively aligned with the target ranking. In this regime, maximizing the SC marginal reward favors candidates whose insertion is more compatible with the target monotone trend than candidates that further disrupt it.
\end{proposition}

\begin{proof}[Proof sketch]
The SC reward measures the change in rank correlation after inserting a candidate pair. When the existing reference set is positively aligned with the target, a candidate that preserves this trend tends to increase the enlarged-set correlation, while a candidate that creates additional inversions tends to decrease it. Thus, under a positive warm-start, the SC reward acts as a local rank-refinement signal. The argument is qualitative rather than a full convergence proof because the practical procedure uses finite reference sets and a drifting optimization target.
\end{proof}

Proposition~\ref{prop:warmstart-interpretation} is the formal role played by the SFT stage in our method: SFT is not assumed to solve the calibration problem, but it initializes the model in a regime where SC refinement is oriented toward the correct monotone structure.


\subsection{Limitations of the idealized analysis}
\label{app:limitations}

We close with explicit acknowledgment of the gap between the idealized population analysis and the practical DPO-based training procedure.

\paragraph{Finite-sample surrogate.}
The surrogate $\kappa_s^{(K)}(X)$ is a Monte Carlo estimate of $\eta(X)$ and therefore introduces both variance and discretization. Our analysis characterizes the large-$K$ objective, not the exact finite-$K$ one.

\paragraph{Drifting reference set.}
The SC reward is evaluated relative to a reference set $(C,K)$ that evolves during training. The propositions above explain the \emph{directional role} of this reward, but do not constitute a full dynamical-systems analysis of the coupled evolution of the model and the reference set.

\paragraph{DPO approximation.}
The practical algorithm does not directly maximize expected SC reward. Instead, SC is used to construct preference pairs, and DPO performs a smoothed preference-optimization step. Thus, our theory should be interpreted as describing the objective that DPO is intended to approximate, rather than as a literal proof of convergence of the training algorithm.


\medskip

In summary, the theoretical analysis supports three claims: (i) the ensemble surrogate consistently estimates the prompt-level target $\eta(X)$; (ii) the SC reward is naturally interpreted as a marginal improvement in a global rank objective; and (iii) a positive warm-start is important because it places the model in a regime where this global rank-refinement signal is oriented toward the desired monotone solution.

\section{Extra results}
\label{app:extra_res}

\paragraph{Analysis.}
As shown in Table~\ref{tab:logiqa}, for LogiQA 2.0 in distribution evaluation, our method achieves the strongest or near-strongest performance across all three foundation models, especially on the ranking-based metrics that are most closely aligned with our objective. For LLaMA-3 8B, Ours-ORCE obtains the highest Spearman correlation and the lowest AURC and E-AURC, indicating substantially better failure prediction than all baselines. On Qwen 7B, Ours-ORCE again achieves the best ranking performance and the best risk--coverage trade-off, while ConfTuner yields the lowest ECE, suggesting that our method is particularly effective at preserving the relative ordering of reliability even when absolute calibration may not always be optimal. On Mistral 7B, Ours-ORCE remains the best method on Spearman correlation, AURC, and E-AURC, although the gains over SFT and Consistency are smaller, indicating that the margin of improvement depends on the underlying model. Across all settings, these results support our central claim that optimizing a global rank-based objective is especially beneficial for failure prediction and failure prediction, while remaining competitive on calibration.

\begin{table*}[t]
\scriptsize
\setlength{\tabcolsep}{5pt}
\centering
\begin{tabular}{llrrrrrr}
\toprule
\multirow{2.5}{*}{\makecell[l]{\textbf{Foundation}\\ \textbf{Models}}} & \multirow{2.5}{*}{\textbf{Methods}} & \multirow{2.5}{*}{\textbf{ECE} $\downarrow$}  & \multicolumn{2}{c}{\textbf{Spearman Correlation}} & \multirow{2.5}{*}{\textbf{AURC} $\downarrow$} & \multirow{2.5}{*}{\textbf{E-AURC} $\downarrow$} & \multirow{2.5}{*}{\textbf{Accuracy} $\uparrow$}  \\
\cmidrule(lr){4-5} 
&  &  & correlation $\uparrow$   & $p$\_value $\downarrow$       \\
\midrule
\multirow{8}{*}{LLAMA-3 8B}  & Vanilla  & 0.291 & 0.017 & 0.498 & 0.503 & 0.346 & 0.496  \\
&Consistency & 0.188 & 0.293 & $1.643\times 10^{-32} $ & 0.384 & 0.214 & 0.478        \\
& Top-k & 0.448 & 0.088 & $4.708 \times 10^{-5}$ & 0.501 & 0.326 & 0.470      \\
& CoT  & 0.407 & 0.206 & $1.767 \times 10^{-16}$ & 0.538 & 0.312 & 0.408      \\
& ConfTuner & 0.063 & 0.374 & $2.718\times 10^{-53}$ &  0.346 & 0.190 & 0.496  \\
& SFT & \underline{0.045} & \underline{0.399} & $4.415\times 10^{-61}$ & \underline{0.326} & \underline{0.170} & \textbf{0.496}    \\
& \bf Ours-ORCE & \textbf{0.057} & \textbf{0.467} & $5.169 \times 10^{-86}$ & \textbf{0.306} & \textbf{0.149} & \textbf{0.496}  \\
\midrule
\multirow{8}{*}{Qwen3 8B}  & Vanilla & 0.250 & 0.154 & $7.519 \times 10^{-10}$ & 0.251 & 0.200 & 0.699\\
&Consistency & 0.220 & 0.360 & $3.390 \times 10^{-49} $ & 0.210 & 0.162 & \underline{0.706}     \\
& Top-k   & 0.195 & 0.125 & $6.361 \times 10^{-7}$ & 0.257 & 0.206 & 0.697  \\
& CoT & 0.126 & 0.160 & $1.792 \times 10^{-10}$ & 0.243 & 0.197 & \textbf{0.710}       \\
& ConfTuner & \textbf{0.039} & 0.397 & $1.227 \times 10^{-60}$ & 0.151 & 0.101 & 0.699 \\
& SFT & 0.094 & \underline{0.410} & $7.000\times 10^{-65}$ & \underline{0.161} & \underline{0.110} & 0.699   \\
& \bf Ours-ORCE & \underline{0.081} & \textbf{0.454} & $1.232 \times 10^{-80}$ & \textbf{0.148} & \textbf{0.098} & 0.699  \\
\midrule
\multirow{8}{*}{Mistral 7B}  & Vanilla & 0.498 & -0.032 & 0.207 & 0.492 & 0.335 & \textbf{0.496}    \\
&Consistency  & 0.163 & \underline{0.304} & $4.767 \times 10^{-35}$ & 0.386 & \underline{0.205} & 0.462        \\
& Top-k  & 0.448 & 0.034 & 0.181 & 0.511 & 0.338 & 0.473             \\
& CoT  & 0.416 & 0.105 & $3.21 \times 10^{-5}$ & 0.576 & 0.341 & 0.398              \\
& ConfTuner & 0.126 & 0.177 & $1.758 \times 10^{-12}$ & 0.439 & 0.283 & \textbf{0.496}\\
& SFT & \underline{0.059} & 0.303 & $8.419 \times 10^{-35}$ & \underline{0.360} & 0.203 & \textbf{0.496}     \\
& \bf Ours-ORCE & \textbf{0.070} & \textbf{0.314} & $2.700\times 10^{-37}$ & \textbf{0.355} & \textbf{0.199} & \textbf{0.496}    \\
\bottomrule
\end{tabular}
\caption{Alignment performance of methods across the foundation models on the LogiQA dataset. }
\label{tab:logiqa}
\end{table*}

\section{Experiment details}
\label{app:exper_details}

\paragraph{Benchmarks}
We conduct experiments on MMLU, DROP, LogiQA, and ReClor to evaluate different aspects of our method. MMLU~\citep{hendrycksmeasuring} covers a wide range of academic and professional subjects and tests whether our method improves calibration for knowledge-intensive multiple-choice reasoning. DROP~\citep{dua2019drop} requires models to perform discrete reasoning over passages, making it suitable for evaluating confidence estimation under passage-grounded numerical and compositional reasoning. For logical reasoning, we use LogiQA 2.0~\citep{liu2023logiqa} for training and evaluate on ReClor~\citep{yureclor} in the main text, allowing us to assess out-of-domain generalization of the proposed order-aware objective when correctness depends on multi-step inference rather than factual recall. We report in-domain LogiQA 2.0 testing results in the Appendix~\ref{app:extra_res}. Together, these benchmarks test whether our decoupled confidence-estimation framework generalizes across knowledge-intensive, discrete reasoning, and logical-reasoning settings while preserving answer accuracy.

\paragraph{Evaluation metrics.}
We evaluate verbalized confidence from both calibration and ranking perspectives. \textbf{Expected Calibration Error (ECE)}~\cite{naeini2015obtaining} measures how well the model's stated confidence matches empirical accuracy by grouping predictions into confidence bins and computing the weighted average gap between confidence and accuracy; lower ECE indicates better calibration. \textbf{Spearman correlation}~\cite{spearman1904proof} evaluates whether confidence scores correctly preserve the relative ordering of response correctness likelihood, measuring the rank correlation between verbalized confidence and correctness labels or surrogate correctness scores; higher Spearman correlation indicates better order alignment. \textbf{Area Under the Risk--Coverage Curve (AURC)}~\cite{geifman2017selective} evaluates failure prediction performance by sorting examples according to confidence and measuring the error rate as low-confidence examples are rejected; lower AURC means the model can better identify unreliable predictions. \textbf{Excess AURC (EAURC)}~\cite{geifmanbias} subtracts the optimal AURC achievable at the same accuracy from the model's AURC, providing an accuracy-normalized measure of failure prediction quality; lower EAURC indicates better confidence-based ranking beyond what is explained by accuracy alone. Together, these metrics assess absolute calibration, relative confidence ordering, and the usefulness of confidence for abstention.

\paragraph{Baselines.}
We compare our method against several representative confidence elicitation and alignment baselines. \textbf{Vanilla} uses our decoupled answer-conditioned prompting strategy without any training or fine-tuning: the model first generates an answer, and then reports verbalized confidence conditioned on the fixed question--answer pair. \textbf{Self-Consistency} follows prior confidence elicitation work~\citep{xiong2024can} by sampling multiple responses for the same question and estimating confidence from their agreement. \textbf{Top-$k$} follows the calibrated elicitation strategy of Tian et al.~\citep{tian2023just}, where the model is prompted to consider multiple likely answers before reporting its confidence. \textbf{CoT} uses chain-of-thought~\cite{wei2022chain} prompting before eliciting confidence, as studied in Xiong et al.~\citep{xiong2024can}, to test whether explicit reasoning improves confidence calibration. \textbf{ConfTuner}~\citep{liconftuner} is a fine-tuning baseline that trains LLMs to express verbalized confidence using a tokenized Brier-score objective. For a fair comparison, we adapt ConfTuner to our decoupled answer conditioned setting, so that it differs from our method only in the confidence alignment objective. Specifically, we apply the ConfTuner loss only when training the verbalized confidence generator, while keeping the generated answer fixed. During inference, the adapted ConfTuner baseline takes the question with the original model's response as input and predicts only the verbalized confidence, without modifying the answer itself.
Finally, \textbf{SFT} denotes a supervised fine-tuning baseline trained directly on the sampling-based correctness surrogate used in our method. This baseline isolates the effect of our order-aware reinforcement learning objective from simply fitting surrogate confidence targets. Together, these baselines cover prompting-based elicitation, sampling-based confidence estimation, and fine-tuning-based verbalized confidence calibration.

\subsection{Compute resources}
\label{app:com_res}
All experiments are conducted on a compute node equipped with 128 CPU cores (Intel Xeon Platinum 8462Y+) and 4 NVIDIA H100 GPUs with 80GB HBM3 memory each.

\subsection{Experimental setting}
\label{app:exp_set}
We implement our method using the OpenRLHF framework. We first train the supervised fine-tuning (SFT) model for 3 epochs and then further optimize it with reinforcement learning for 3 epochs. For both MMLU and DROP, we randomly sample 10{,}000 training examples for SFT and 20{,}000 training examples for RL. For LogiQA 2.0, we use one-third of the training set for SFT and the remaining two-thirds for RL.

\section{Broader impact}
\label{app:broader_impact}
This work aims to improve the reliability of language models by aligning verbalized confidence with answer correctness likelihood. A potential positive societal impact is that better-calibrated confidence estimates can support safer model deployment, for example by improving failure prediction, failure prediction, abstention, and routing decisions. At the same time, this work also carries potential risks. Users may over-trust verbalized confidence even when it remains imperfect, and miscalibrated confidence estimates could be especially harmful in high-stakes settings such as healthcare, law, or education if used without appropriate human oversight. We therefore view this method as a tool for improving model reliability, but not as a substitute for domain-specific safeguards or human judgment in consequential applications.


\end{document}